\documentclass[letterpaper]{article}
\usepackage{times}
\usepackage{helvet}
\usepackage{courier}
\usepackage{pifont}
\usepackage{etoolbox}

\newtoggle{fullpaper} 
\toggletrue{fullpaper}
\newcommand{\cmark}{\ding{51}}
\newcommand{\xmark}{\ding{55}}
\newcommand{\instA}[0]{\ensuremath{^{\dagger}}}
\newcommand{\instB}[0]{\ensuremath{^{\star}}}
\newcommand{\instC}[0]{\ensuremath{^{\ddagger}}}
\newcommand{\instD}[0]{\ensuremath{^{\S}}}
\usepackage{times}
\usepackage{graphicx} 
\usepackage{subfigure} 
\usepackage{amssymb}
\usepackage{mathrsfs} 

\usepackage{natbib}
\usepackage{amsmath}
\usepackage[T1]{fontenc}
\usepackage{amsthm}
\newtheorem{thm}{Theorem}
\newtheorem{lemma}{Lemma}
\newtheorem{cor}{Corollary}
\newtheorem{assumption}{Assumption}
\newtheorem{defn}{Definition}
\newtheorem{prop}{Proposition}
\newtheorem{rem}{Remark}
\newtheorem{ex}{Example}

\iftoggle{fullpaper}{
\newtheorem{appxlem}{Lemma}[section]
}{}
\usepackage{url}
\usepackage{algorithm}
\usepackage{algorithmic}
\usepackage{color}
\usepackage{xspace}

\newcommand{\I}[0]{\ensuremath{\mathcal{I}}\xspace}
\newcommand{\X}[0]{\ensuremath{\mathcal{X}}\xspace}

\newcommand{\lih}[2]{\ensuremath{p_{#1}\left({#2}\right)}\xspace}
\newcommand{\pri}[1]{\ensuremath{\xi\left({#1}\right)}\xspace}
\newcommand{\data}[0]{\ensuremath{D}\xspace}
\newcommand{\field}[1]{\ensuremath{\mathfrak{S}_{#1}}}
\newcommand{\lap}[1]{\ensuremath{\mathrm{Lap}\left(#1\right)}}
\newcommand{\bet}[2]{\ensuremath{\mathrm{Beta}\left({#1},{#2}\right)}}
\newcommand{\term}[1]{\emph{#1}}
\renewcommand{\v}[1]{\ensuremath{\boldsymbol{#1}}}
\renewcommand{\Pr} {\ensuremath{\mathbb{P}}}
\newcommand{\E} {\ensuremath{\mathbb{E}}}
\newcommand{\cf}[0]{\emph{cf.}\xspace}
\newcommand{\eg}[0]{\emph{e.g.},\xspace}
\newcommand{\ie}[0]{\emph{i.e.},\xspace}

\newcommand{\citenop}[1]{\citeauthor{#1}~\citeyear{#1}\xspace}

\newcommand \Bernoulli {\textrm{Bernoulli}}
\newcommand \Beta {\textrm{Beta}}

\newcommand \CD {\mathcal{D}}
\newcommand \CY {\mathcal{Y}}

\newcommand \Bay {\ensuremath{\mathscr{B}}}
\newcommand \Adv {\ensuremath{\mathscr{A}}}

\newcommand \Params {\Theta}
\newcommand \param {\theta}
\newcommand \vparam {\v{\theta}}

\newcommand \util{u}

\newcommand \family {\mathcal{F}_{\Params}}
\newcommand \bel {\xi}

\newcommand \ctab {\ensuremath{h}\xspace}
\newcommand \dclose {\ensuremath{\mathcal{N}_\I}\xspace}
\newcommand{\extpar}[1]{\overline{\pi}_{#1}}

\newcommand{\cset}[2]{\left\{\, #1 \mathrel{:} #2 \,\right\} }

\newcommand \abs[1] {\left|#1\right|}
\newcommand{\proj}[2]{\ensuremath{\mathrm{C}^{#1}({#2})}\xspace}
\newcommand{\bigoh}[1]{\mathcal{O}\left({#1}\right)}

\newcommand{\indeg}[1]{\mathrm{indeg}\left({#1}\right)}
\newcommand{\prob}[1]{\mathbb{P}\left({#1}\right)}

\newcommand{\iid}[0]{i.i.d.\xspace}

\newcommand\dd{\,\mathrm{d}}

\title{On the Differential Privacy of Bayesian Inference}
\author{Zuhe Zhang\instA \and  Benjamin I. P. Rubinstein\instB \\
	\instA School of Mathematics and Statistics, \\
	\instB Department of Computing and Information Systems, \\
	The University of Melbourne, Australia \\ 
	zhang.zuhe@gmail.com, brubinstein@unimelb.edu.au\\
\And
Christos Dimitrakakis\instC\instD\\
\instC Univ-Lille-3, France \\
\instD Chalmers University of Technology, Sweden  \\
christos.dimitrakakis@gmail.com
}

\begin{document}

\begin{abstract}
We study how to communicate findings of Bayesian inference to third parties, while preserving the strong guarantee of differential privacy. Our main contributions are four different algorithms for private Bayesian inference on probabilistic graphical models. These include two mechanisms for adding noise to the Bayesian updates, either directly to the posterior parameters, or to their Fourier transform so as to preserve update consistency. We also utilise a recently introduced posterior sampling mechanism, for which we prove bounds for the specific but general case of discrete Bayesian networks; and we introduce a maximum-a-posteriori private mechanism. Our analysis includes utility and privacy bounds, with a novel focus on the influence of graph structure on privacy. Worked examples and experiments with Bayesian na\"ive Bayes and Bayesian linear regression illustrate the application of our mechanisms.
\end{abstract}




\section{Introduction}

We consider the problem faced by a statistician \Bay{} who analyses data and communicates her findings to a third party \Adv{}. While \Bay{} wants to learn as much as possible \emph{from} the data, she doesn't want \Adv{} to learn \emph{about} any individual datum. This is for example the case where \Adv{} is an insurance agency, the data are medical records, and \Bay{} wants to convey the efficacy of drugs to the agency, without revealing the specific illnesses of individuals in the population.
Such requirements of \emph{privacy} are of growing interest in the learning~\cite{chaudhuri2012convergence,duchi:local-privacy}, theoretical computer science~\cite{dpStats,MechDesign} and databases communities~\cite{Barak2007,zhang2014privbayes}  due to the impact on individual privacy by real-world data analytics.

In our setting, we assume that \Bay{} is using \emph{Bayesian inference} to draw conclusions from observations of a system of random variables by updating a prior distribution on parameters (\ie \emph{latent} variables) to a posterior. Our goal is to release an approximation to the posterior that preserves privacy. We adopt the formalism of \emph{differential privacy} to characterise how easy it is for \Adv{} to discover facts about the \emph{individual} data from the \emph{aggregate} posterior.  Releasing the posterior permits external parties to make further inferences at will. For example, a third-party pharmaceutical might use the released posterior as a prior on the efficacy of drugs, and update it with their own patient data. Or they could form a predictive posterior for classification or regression, all while preserving differential privacy of the original data. 


Our focus in this paper is Bayesian inference in \emph{probabilistic graphical models} (PGMs), which are popular as a tool for modelling conditional independence assumptions. Similar to the effect on statistical and computational efficiency of non-private inference, a central tenet of this paper is that independence structure should impact privacy. \emph{Our mechanisms and theoretical bounds are the first to establish such a link between PGM graph structure and privacy.}


\paragraph{Main Contributions.} 
We develop the first mechanisms for Bayesian inference on the flexible PGM framework (\cf Table~\ref{tab:summary}). We propose two posterior perturbation mechanisms for networks with likelihood functions from exponential families and conjugate priors, that add Laplace noise~\cite{Dwork06} to posterior parameters (or their Fourier coefficients) to preserve privacy. The latter achieves stealth through consistent posterior updates. For general Bayesian networks, posteriors may be non-parametric. In this case, we explore a mechanism~\cite{alt:robust} which samples from the posterior to answer queries---no additional noise is injected. 
We complement our study with a maximum \emph{a posteriori} estimator that leverages the exponential mechanism~\cite{MechDesign}. Our utility and privacy bounds connect privacy and graph/dependency structure, and are complemented by illustrative experiments with Bayesian na\"ive Bayes and linear regression.




\paragraph{Related Work.}

Many individual learning algorithms have been adapted to maintain differential privacy, including regularised logistic regression~\cite{ChaudhuriM08}, the SVM~\cite{privateSVM,dpERM}, PCA~\cite{NIPS2012_4565}, the functional mechanism~\cite{zhang2012functional} and trees~\cite{5360518}.


Probabilistic graphical models have been used to preserve privacy. \citet{zhang2014privbayes} learned a graphical model from data, in order to generate \emph{surrogate data} for release; while \citet{ProbInfDP} fit a model to the response of private mechanisms to clean up output and improve accuracy. \citet{xiao2012bayesian} similarly used Bayesian credible intervals to increase the utility of query responses. 

Little attention has been paid to private inference in the Bayesian setting. We seek to adapt Bayesian inference to preserve differential privacy when releasing posteriors. \citeauthor{alt:robust}~(\citeyear{alt:robust};~\citeyear{arxiv:robust}) introduce a differentially-private mechanism for Bayesian inference based on posterior sampling---a mechanism on which we build---while \citet{472366} considers further refinements.
\citet{DBLP:journals/corr/WangFS15} explore Monte Carlo approaches to Bayesian inference using the same mechanism, while \citet{mir2012differentially} was the first to establish differential privacy of the Gibbs estimator~\cite{MechDesign} by minimising risk bounds.

This paper is the first to develop mechanisms for differential privacy under the general framework of Bayesian inference on multiple, dependent r.v.'s. Our mechanisms consider graph structure and include a purely Bayesian approach that only places conditions on the prior. We show how the (stochastic) Lipschitz assumptions of \citet{alt:robust} lift to graphs of r.v.'s, and bound KL-divergence when releasing an empirical posterior based on a modified prior. While~\citet{dpERM} achieve privacy in regularised Empirical Risk Minimisation through objective randomisation, we do so through conditions on priors. We develop an alternate approach that uses the additive-noise mechanism of~\citet{Dwork06} to perturb posterior parameterisations; and we apply techniques due to~\citet{Barak2007}, who released marginal tables that maintain consistency in addition to privacy, by adding Laplace noise in the Fourier domain. Our motivation is novel: we wish to guarantee privacy against omniscient attackers and stealth against unsuspecting third parties.

\begin{table*}
\centering
\resizebox{1.0\linewidth}{!}{
	\begin{tabular}{| l | c | p{5.0cm} | p{4.0cm} | l |}
    \hline
    &\textbf{DBN only} & \textbf{Privacy} & \textbf{Utility type}  & \textbf{Utility bound}\\
    \hline
    Laplace  & \cmark & $(\epsilon, 0)$ & closeness of posterior&$\bigoh{m n\ln n}\left[1-\exp\left(-\frac{n\epsilon}{2|\I|}\right)\right]+\sqrt{-\bigoh{m n\ln n}\ln\delta}$ \\ \hline
    Fourier & \cmark & $(\epsilon, 0)$ & close posterior params & $\frac{4|\dclose|}{\epsilon} \left(2^{|\pi_i|}\log\frac{|\dclose|}{\delta} + t |\dclose|\right)$\\ \hline
    Sampler & \xmark & $(2L, 0)$ if Lipschitz; or\ \ \ \linebreak[4] $(0, \sqrt{M/2}$) stochastic Lipschitz & expected utility functional wrt posterior & $\bigoh{\eta + \sqrt{\ln (1/\delta)/N}}$ \cite{arxiv:robust}\\ \hline
    MAP & \xmark & $(\epsilon, 0)$ & closeness  of MAP & $\Pr(S^{c}_{2t})\leq \exp(-\epsilon t)/\pri{S_{t}}$\\
    \hline
    \end{tabular}
}
\caption{Summary of the privacy/utility guarantees for this paper's mechanisms. See below for parameter definitions.}
\label{tab:summary}
\end{table*}




\section{Problem Setting}
Consider a Bayesian statistician \Bay{} estimating the parameters $\vparam$ of some family of distributions $\family = \cset{p_{\vparam}}{\vparam \in \Params}$ on a system of r.v.'s $\v{X}=\cset{X_i}{i \in \I}$, where $\I$ is an index set, with
observations denoted $x_i\in \X_i$, where $\X_{i}$ is the sample space of $X_{i}$.
\Bay{} has a prior distribution\footnote{Precisely, a probability measure on a $\sigma$-algebra $(\Params_i,\field{\Params_i})$.} $\bel$ on $\Params$ reflecting her prior belief, which she updates on an observation $\v{x}$ to obtain posterior
\begin{align*}
\bel(B \mid \v{x})&=\frac{\int_{B}\lih{\vparam}{\v{x}} \dd\pri{\vparam}}{\phi(\v{x})},\enspace \forall B\in\field{\Params}
\end{align*}
where $\phi(\v{x})\triangleq\int_{\Params}\lih{\vparam}{\v{x}}d\pri{\vparam}$. 
Posterior updates are iterated over an \iid dataset $\data\in \mathcal{D}=(\prod_{i} \X_{i})^{n}$ to
$\bel(\cdot \mid \data)$.

\Bay{}'s goal is to communicate her posterior distribution to a third party \Adv{}, while limiting the information revealed about the original data.
From the point of view of the data provider, \Bay{} is a trusted party.\footnote{Cryptographic tools for untrusted \Bay{} do not prevent information leakage to \Adv{} \cf \eg  \cite{pagning:indocrypt:2014}.}
 However, she may still inadvertently reveal information.
We assume that \Adv{} is computationally unbounded, and has knowledge of the prior $\bel$ and the family $\family$.
To guarantee that \Adv{} can gain little additional information about $\data$ from their communication, \Bay{} uses Bayesian inference to learn from the data, and a differentially-private posterior to ensure disclosure to \Adv{} is carefully controlled.


\subsection{Probabilistic Graphical Models}
Our main results focus on PGMs which model conditional independence assumptions with joint factorisation
\[
p_{\vparam}(\v{x}) = \prod_{i \in \I} \lih{\vparam}{x_i\mid x_{\pi_i}},
\quad
x_{\pi_i} = \cset{x_j}{j \in \pi_i}\enspace,
\]
where 
$\pi_i$ are the parents of the $i$-th variable in a Bayesian network---a
directed acyclic graph with r.v.'s as nodes. 

\begin{ex}\label{ex:boolean}
For concreteness, we illustrate some of our mechanisms on systems of Bernoulli r.v.'s $X_i\in\{0,1\}$. In that case, we  represent the conditional distribution of $X_i$ given its parents as Bernoulli with parameters $\param_{i,j} \in [0,1]$ :
\[
(X_i \mid X_{\pi_i} = j) \sim \Bernoulli(\param_{i,j})\enspace.
\]
The choice of conjugate prior $\bel(\vparam) = \prod_{i,j} \bel_{i,j}(\param_{i,j})$ has Beta marginals with parameters $\alpha_{i,j}, \beta_{i,j}$, so that:
\[
	(\param_{i,j} \mid \alpha_{i,j} = \alpha, \beta_{i,j} = \beta)
\sim
\Beta(\alpha, \beta)\enspace.
\]
Given observation $\v{x}$, the updated posterior Beta parameters are $\alpha_{i,j} := \alpha_{i,j} + x_{i}$ and $\beta_{i,j} := \beta_{i,j} + (1 - x_{i})$ if $x_{\pi_i} = j$.
\label{ex:beta}
\end{ex}

\subsection{Differential Privacy}
\Bay{} communicates to \Adv{} by releasing information about the posterior distribution, via randomised mechanism $M$ that maps dataset $\data\in\mathcal{D}$ to a response in set $\CY$. \citet{Dwork06} characterise when such a mechanism is private:
\begin{defn}[Differential Privacy]
A randomised mechanism $M : \CD \to \CY $ is $(\epsilon, \delta)$-DP if for any \emph{neighbouring}  $\data, \tilde{\data} \in \CD$, and measurable $B \subseteq \CY$:
\[
\Pr[M(\data)\in B]\leq e^{\epsilon} \Pr[M(\tilde{\data})\in B] + \delta,
\]
where $\data = (\v{x}^i)_{i=1}^n$, $\tilde{\data} = (\tilde{\v{x}}^i)_{i=1}^n$ are neighbouring if $\v{x}^i \neq \tilde{\v{x}}^i$ for at most one $i$.
\label{def:differential-privacy}
\end{defn}
This definition requires that neighbouring datasets induce similar response distributions. Consequently, it is impossible for \Adv{} to identify the true dataset from bounded mechanism query responses. Differential privacy assumes no bounds on adversarial computation or auxiliary knowledge.




\section{Privacy by Posterior Perturbation}

One approach to differential privacy is to use additive Laplace noise~\citep{Dwork06}. Previous work has focused on the addition of noise directly to the outputs of a non-private mechanism. We are the first to apply Laplace noise to the posterior parameter updates.

\subsection{Laplace Mechanism on Posterior Updates}
Under the setting of Example~\ref{ex:beta}, we can add Laplace noise to the posterior parameters.  Algorithm~\ref{alg:laplace} releases perturbed parameter updates for the Beta posteriors, calculated simply by \emph{counting}.
\begin{algorithm}[htb]
  \begin{algorithmic}[1]
    \STATE \textbf{Input} data \data; graph $\I, \{\pi_i\mid i\in\I\}$; parameter $\epsilon>0$
    \STATE calculate posterior updates: $\Delta\alpha_{i,j}, \Delta\beta_{i,j}$ for all \\ $i\in\I, j\in \{0,1\}^{|\pi_{i}|}$
    \STATE perturb updates: $\Delta\alpha'_{i,j} \triangleq \Delta\alpha_{i,j} + \lap{\frac{2|\I|}{\epsilon}}$, \\ $\Delta\beta'_{i,j} \triangleq \Delta\beta_{i,j} + \lap{\frac{2|\I|}{\epsilon}}$.
    \STATE truncate: $Z^{(1)}_{i,j} \triangleq \mathbf{1}_{[0,n]}(\Delta\alpha'_{i,j}), Z^{(2)}_{i,j} \triangleq \mathbf{1}_{[0,n]}(\Delta\beta'_{i,j})$
    \STATE output $\v{Z}_{i,j}=(Z^{(1)}_{i,j}, Z^{(2)}_{i,j})$
  \end{algorithmic}
  \caption{Laplace Mechanism on Posterior Updates \label{alg:laplace}}
\end{algorithm}
It then adds zero-mean Laplace-distributed noise to the updates $\Delta\v{\omega}=(\cdots, \Delta\alpha_{i,j}, \Delta\beta_{i,j}, \cdots)$. This is the final dependence on \data. Finally, the perturbed updates $\Delta\v{\omega}'$ are truncated at zero to rule out invalid Beta parameters and are upper truncated at $n$. This yields an upper bound on the raw updates and  facilitates an application of McDiarmid's bounded-differences inequality
\iftoggle{fullpaper}{(\cf Lemma~\ref{McDiarmid} in the Appendix)}{(\cf full report~\citenop{arxiv:full})}
in our utility analysis. Note that this truncation only improves utility (relative to the utility pre-truncation), and does not affect privacy.

\paragraph{Privacy.}
To establish differential privacy of our mechanism, we must calculate a Lipschitz condition for the vector $\v{\Delta\omega}$ called \term{global sensitivity}~\cite{Dwork06}.
\begin{lemma}\label{lemma:laplace-gs}
	For any \emph{neighbouring} datasets $\data,\tilde{\data}$, the corresponding updates $\Delta\v{\omega},\Delta\tilde{\v{\omega}}$ satisfy $\left\|\Delta\v{\omega}-\Delta\tilde{\v{\omega}}\right\|_1\leq 2|\I|$.
\end{lemma}
\begin{proof}
By changing the observations of one datum, at most two counts associated with each $X_{i}$ can change by 1.
\end{proof}

\begin{cor}
Algorithm~\ref{alg:laplace} preserves $\epsilon$-differential privacy.
\end{cor}
\begin{proof}
	Based on Lemma~\ref{lemma:laplace-gs}, the intermediate $\Delta\v{\omega}'$ preserve $\epsilon$-differential privacy~\cite{Dwork06}. Since truncation depends only on $\Delta\v{\omega}'$, the $\v{Z}$ preserves the same privacy.
\end{proof}



\paragraph{Utility on Updates.}
Before bounding the effect on the posterior of the Laplace mechanism, we demonstrate a utility bound on the posterior update counts.

\begin{prop}\label{prop:laplace-util-updates}
	With probability at least $1-\delta$, for $\delta\in(0,1)$, the update counts computed by Algorithm~\ref{alg:laplace} are close to the non-private counts
	$$\left\|\Delta\v{\omega}-\Delta\v{\omega}'\right\|_{\infty}\leq \frac{{2|\I|}}{\epsilon}\ln\left(\frac{2m}{\delta}\right)\enspace,$$
where $m=\sum_{i\in I}2^{|\pi_{i}|}.$
\end{prop}
This bound states that w.h.p., none of the updates can be perturbed beyond $O(|\I|^{2}/\epsilon)$.
This implies the same bound on the deviation between $\Delta\v{\omega}$ and the revealed truncated $\v{Z}$.

\paragraph{Utility on Posterior.}
We derive our main utility bounds for Algorithm~\ref{alg:laplace} in terms of posteriors, proved in the
\iftoggle{fullpaper}{Appendix}{full report~\cite{arxiv:full}}.
We abuse notation, and use $\bel$ to refer to the prior density; its meaning will be apparent from context.
Given priors
$\bel_{i,j}(\theta_{i,j})=\bet{\alpha_{i,j}}{\beta_{i,j}}$,
the posteriors on $n$ observations are
\begin{equation}
\bel_{i,j}(\theta_{i,j}|D)=\Beta(\alpha_{i,j}+\Delta\alpha_{i,j},\beta_{i,j}+\Delta\beta_{i,j})\enspace.\nonumber
\end{equation}
The privacy-preserving posterior parametrised by the output of Algorithm~\ref{alg:laplace} is $$\bel'_{i,j}(\theta_{i,j}|D)=\Beta\left(\alpha_{i,j}+Z^{(1)}_{i,j},\beta_{i,j}+Z^{(2)}_{i,j}\right)\enspace.$$

It is natural to measure utility by the KL-divergence between the joint product posteriors $\bel(\v{\theta}|D)$ and $\bel'(\v{\theta}|D)$, which is the sum of the component-wise divergences, with each having known closed form.
In our analysis, the divergence is a random quantity, expressible as the sum
$\sum_{i,j}^{m}f_{i,j}(\v{Z}_{i,j})$, where the randomness is due to the added noise.
We demonstrate this r.v. is not too big, w.h.p.   
\begin{thm}\label{thm:McDiarmid}
Let $m=\sum_{i\in \I}2^{|\pi_{i}|}$. Assume that $Z_{i,j}$ are independent and $f$ is a mapping from $\mathcal{Z}^{m}$ to $\mathbb{R}$: $f(\cdots, \v{z}_{i,j}, \cdots)\triangleq\sum_{i,j} f_{i,j}(\v{z}_{i,j})$.
Given $\delta>0$, we have
\begin{align*}
\Pr\left[f(\v{Z})\geq \E(f(\v{Z}))+\left(-\frac{1}{2}\sum_{i,j}c_{i,j}\ln\delta\right)^{\frac{1}{2}}\right]\leq \delta
\end{align*}
where $c_{i,j}\leq (2n+1)\ln[(\alpha_{i,j}+n+1)+(\beta_{i,j}+n+1))$ and $\E(f_{i,j}(\v{Z}_{i,j})]\leq n\ln((\alpha_{i,j}+\Delta\alpha_{i,j})(\beta_{i,j}+\Delta\beta_{i,j}))=\mathcal{U}$. \\
Moreover, when $n\geq b=\frac{2|\I|}{\epsilon}$, the bound for expectation can be refined as the following
\begin{align*}
\ln[(\alpha_{i,j}+n+1)(\beta_{i,j}+n+1)]\left(\frac{n}{2}\exp\left(-\frac{n\epsilon}{2|\I|}\right)\right)\enspace.
\end{align*}
The loss of utility measured by KL-divergence is no more than
$$\bigoh{m n\ln n}\left[1-\exp\left(-\frac{n\epsilon}{2|\I|}\right)\right]+\sqrt{-\bigoh{m n\ln n}\ln\delta}$$
with probability at least $1-\delta$.
\end{thm}
Note that $m$ depends on the structure of the network: bounds are better for networks with an underlying graph having smaller average in-degree.




\subsection{Laplace Mechanism in the Fourier Domain}

Algorithm~\ref{alg:laplace} follows \term{Kerckhoffs's Principle}~\cite{Kerckoffs} of ``no security through obscurity'': differential privacy defends against a mechanism-aware attacker. However \emph{additional stealth} may be required in certain circumstances. An oblivious observer will be tipped off to our privacy-preserving activities by our independent perturbations, which are likely inconsistent with one-another (\eg noisy counts for $X_1,X_2$ and $X_2,X_3$ will say different things about $X_2$). To achieve differential privacy and stealth, we turn to \citet{Barak2007}'s study of consistent marginal contingency table release. This section presents a particularly natural application to Bayesian posterior updates. 

Denote by $\ctab\in\mathbb{R}^{\{0,1\}^{|\I|}}$ the \term{contingency table} over r.v.'s \I induced by \data: \ie for each combination of variables $j\in\{0,1\}^{|\I|}$, component or \term{cell} $\ctab_j$ is a non-negative count of the observations in \data with characteristic $j$. Geometrically \ctab is a real-valued function over the $|\I|$-dimensional Boolean hypercube. Then the parameter delta's of our first mechanism correspond to cells of $(|\pi_i|+1)$-way marginal contingency tables $\proj{\extpar{i}}{\ctab}$ where vector $\extpar{i}\triangleq\pi_i+e_i$ and the projection/marginalisation operator is defined as
\begin{eqnarray}
	\left(\proj{j}{\ctab}\right)_\gamma &\triangleq& \sum_{\eta : \langle\eta,j\rangle=\gamma} h_\eta\enspace. \label{eq:proj}
\end{eqnarray}
We wish to release these statistics as before, however we will not represent them under their Euclidean coordinates but instead in the Fourier basis $\{f^j : j\in\{0,1\}^{|\I|}\}$ where
\begin{eqnarray*}
	f^j_\gamma &\triangleq& (-1)^{\langle\gamma,j\rangle} 2^{-|\I|/2}\enspace.
\end{eqnarray*}
Due to this basis structure and linearity of the projection operator, any marginal contingency table must lie in the span of few projections of Fourier basis vectors~\cite{Barak2007}:

\begin{thm}
	For any table $\ctab\in\mathbb{R}^{\{0,1\}^{|\I|}}$ and set of variables $j\in\{0,1\}^{|\I|}$, the marginal table on $j$ satisfies $\proj{j}{\ctab}=\sum_{\gamma\preceq j}\left\langle f^\gamma,\ctab\right\rangle\proj{j}{f^\gamma}$.
\end{thm}

This states that marginal $j$ lies in the span of only those (projected) basis vectors $f^\gamma$ with $\gamma$ contained in $j$. The number of values needed to update $X_i$ is then $2^{|\pi_i|+1}$, potentially far less than suggested by~\eqref{eq:proj}. To release updates for two r.v.'s $i,j\in\I$ there may well be significant overlap $\langle\extpar{i}, \extpar{j}\rangle$; we need to release once, coefficients $\langle f^\gamma,\ctab\rangle$ for $\gamma$ in the downward closure of variable neighbourhoods:
\begin{eqnarray*}
	\dclose &\triangleq& \bigcup_{i\in\I}\bigcup_{j\preceq \extpar{i}} j\enspace.
\end{eqnarray*}

\paragraph{Privacy.}
By~\cite[Theorem~6]{Barak2007} we can apply Laplace additive noise to release these Fourier coefficients.

\begin{cor}\label{cor:fourier-privacy}
	For any $\epsilon>0$, releasing for each $\gamma\in \dclose$ the Fourier coefficient $\langle f^\gamma, \ctab\rangle+\lap{2|\dclose|\epsilon^{-1}2^{-|\I|/2}}$ (and Algorithm~\ref{alg:fourier}) preserves $\epsilon$-differential privacy.
\end{cor}

\begin{rem}
	Since $|\dclose|\leq|\I|2^{1+\max_{i\in\I}\indeg{i}}$, at worst we have noise scale
    \[
    |\I| 2^{2+\max_i\indeg{i}-|\I|/2}/\epsilon.
    \] This compares favourably with Algorithm~\ref{alg:laplace}'s noise scale provided no r.v. is child to more than half the graph. Moreover the denser the graph---the more overlap between nodes' parents and the less conditional independence assumed---the greater the reduction in scale. This is intuitively appealing. 
\end{rem}

\paragraph{Consistency.} What is gained by passing to the Fourier domain, is that the perturbed marginal tables of Corollary~\ref{cor:fourier-privacy} are consistent: anything in the span of projected Fourier basis vectors corresponds to some valid contingency table on $\I$ with (possibly negative) real-valued cells~\cite{Barak2007}.


\begin{algorithm}[htb]
  \begin{algorithmic}[1]
    \STATE \textbf{Input} data \data; graph $\I, \{\pi_i\mid i\in\I\}$; prior parameters $\v{\alpha},\v{\beta}\succeq \v{0}$; parameters $t,\epsilon>0$
    \STATE define contingency table $h\in\mathbb{R}^{\{0,1\}^{|\I|}}$ on \data 
    \STATE define downward closure $\dclose=\bigcup_{i\in\I}\bigcup_{j\preceq \extpar{i}}j$
    \FOR{$\gamma\in\dclose$}
    \STATE Fourier coefficient $z_\gamma=\langle f^\gamma, h\rangle + \lap{\frac{2|\dclose|}{\epsilon 2^{|\I|/2}}}$
    \ENDFOR
    \STATE increment first coefficient $z_{\v{0}}\leftarrow z_{\v{0}} + \frac{4 t |\dclose|^2}{\epsilon 2^{|\I|/2}}$ 
    \FOR{$i\in\I$} 
    \STATE project marginal for $X_i$ as $h^i = \sum_{\gamma\preceq\extpar{i}} z_\gamma \proj{\extpar{i}}{f^\gamma}$
      \FOR{$j\preceq \pi_i$}
      \STATE output posterior param $\left(\alpha_{ij}+ h^i_{e_i+ j}, \beta_{ij}+ h^i_j\right)$ 
      \ENDFOR
    \ENDFOR
  \end{algorithmic}
  \caption{Laplace Mechanism in the Fourier Domain \label{alg:fourier}}
\end{algorithm}

\paragraph{Non-negativity.} So far we have described the first stage of Algorithm~\ref{alg:fourier}. The remainder yields \term{stealth} by guaranteeing releases that are non-negative w.h.p. 
We adapt an idea of \citet{Barak2007} to increase the coefficient of Fourier basis vector $f^{\v{0}}$, affecting a small increment to each cell of the contingency table. While there is an exact minimal amount that would guarantee non-negativity, it is data dependent. Thus our efficient $\bigoh{|\dclose|}$-time approach is randomised.

\begin{cor}
	For $t>0$, adding $4 t |\dclose|^2 \epsilon^{-1} 2^{-k/2}$ to $f^{\v{0}}$'s coefficient induces a non-negative table w.p. $\geq 1-\exp(-t)$.
\end{cor}

Parameter $t$ trades off between the probability of non-negativity and the resulting (minor) loss to utility. 
In the rare event of negativity, re-running Algorithm~\ref{alg:fourier} affords another chance of stealth at the cost of privacy budget $\epsilon$. We could alternatively truncate to achieve validity, sacrificing stealth but not privacy.

\paragraph{Utility.} Analogous to Proposition~\ref{prop:laplace-util-updates}, each perturbed marginal is close to its unperturbed version w.h.p.

\begin{thm}\label{thm:fourier-utility}
	For each $i\in\I$ and $\delta\in(0,1)$, the perturbed tables in Algorithm~\ref{alg:fourier} satisfy with probability at least $1-\delta$:
	\begin{eqnarray*}
		\left\|\proj{\extpar{i}}{h} - h^i\right\|_1 &\leq& \frac{4|\dclose|}{\epsilon} \left(2^{|\pi_i|}\log\frac{|\dclose|}{\delta} + t |\dclose|\right)\enspace.
	\end{eqnarray*}
\end{thm}

Note that the scaling of this bound is reasonable since the table $h^i$ involves $2^{|\pi_i|+1}$ cells.





\section{Privacy by Posterior Sampling}
For general Bayesian networks, \Bay{} can release samples from the posterior~\cite{alt:robust} instead of perturbed samples of the posterior's parametrisation. We now develop a calculus of building up (stochastic) Lipschitz properties of systems of r.v.'s that are locally (stochastic) Lipschitz. Given smoothness of the entire network, differential privacy and utility of posterior sampling follow.

\subsection{(Stochastic) Lipschitz Smoothness of Networks}
The distribution family $\{p_{\theta}:\theta\in \Theta\}$ on outcome space $\mathcal{S}$, equipped with pseudo metric\footnote{Meaning that $\rho(x,y)=0$ does not necessarily imply $x=y$.} $\rho$, is \term{Lipschitz continuous} if
\begin{assumption}[Lipschitz Continuity]
\label{ass:hoelder-observations}
Let $d(\cdot, \cdot)$ be a metric on $\mathbb{R}$. There exists $L>0$ such that, for any $\theta\in \Theta$:
\begin{align*}
d(p_{\theta}(x), p_{\theta}(y)) \leq L\rho(x,y), \forall x,y\in \mathcal{S}.
\end{align*}
\end{assumption}
We fix the distance function $d$ to be the absolute log-ratio (\cf differential privacy).
Consider a general Bayesian network. The following lemma shows that the individual Lipschitz continuity of the conditional likelihood at every $i\in{\I}$ implies the global Lipschitz continuity of the network.
\begin{lemma}\label{lemma:Lip-Con}
	If there exists $\v{L}=(L_{1}, \cdots, L_{|\I|})\geq\v{0}$ such that $\forall i\in \I$, $\forall\v{x}, \v{y} \in \X= \prod^{|\I|}_{i=1} \X_{i}$ we have $d(p_{\v{\theta}}(x_{i}|x_{\pi_{i}}), p_{\v{\theta}}(y_{i}|y_{\pi_{i}}))\leq L_{i}\rho_{i}(x_{i},y_{i})$,
then $d(p_{\v{\theta}}(\v{x}), p_{\v{\theta}}(\v{y})) \leq \| \v{L}\|_{\infty} \v{\rho}(\v{x},\v{y})$ where $\v{\rho}(\v{x},\v{y})=\sum^{|I|}_{i=1}\rho_{i}(x_{i}, y_{i})$.
\end{lemma}
Note that while Lipschitz continuity holds uniformly for some families \eg the exponential distribution, this is not so for many useful distributions such as the Bernoulli. In such cases a relaxed assumption requires that the prior be concentrated on smooth regions.
\begin{assumption}[Stochastic Lipschitz Continuity]
\label{ass:hoelder-measure-observations}
Let the set of $L$-Lipschitz $\param$ be $$\Theta_{L}\triangleq \left\{\theta\in \Theta: \sup_{x,y\in \mathcal{S}}\{d(p_{\theta}(x), p_{\theta}(y)) - L\rho(x, y)\} \leq 0 \right\}$$ Then there exists constants $c, L_{0}>0$ such that, $\forall L\geq L_{0}$: $\xi(\Theta_{L})\geq 1-e^{-cL}.$
\end{assumption}
\begin{lemma}\label{lemma:Sto-Lip-Con}
	For the conditional likelihood at each node $i\in\I$, define the set ${\Theta}_{i, L}$ of parameters for which Lipschitz continuity holds with Lipschitz constant $L$.
If $\exists \v{c}=(c_{1}, \cdots, c_{|\I|})$ such that $\forall i, L\geq L_{0}$,
$\xi(\Theta_{i, L})\geq 1-e^{-c_{i}L}$, then
$\xi(\Theta_{L})\geq 1-e^{-c'L},$ where $c'=\min_{i\in \I} c_{i}-\ln{|\I|}/L_{0}$ when $|\I|\leq e^{L_0\min_{i\in\I} c_{i}}$.
\end{lemma}
Therefore, \cite[Theorem 7]{arxiv:robust} asserts differential privacy of the Bayesian network's posterior.
\begin{thm}\label{thm:sampler-dp}
  Differential privacy is satisfied using the log-ratio distance, for all $B\in \field{\Params}$ and $\v{x},\v{y}\in\v{\X}$:
  \begin{enumerate}
	  \item Under the conditions in Lemma~\ref{lemma:Lip-Con}:
    $$\bel(B \mid \v{x}) \leq \exp\{ 2L \v{\rho}(\v{x},\v{y}) \} \bel(B \mid \v{y})$$
    \ie the posterior $\bel$ is
    $(2 \| \v{L}\|_{\infty}, 0)$-differentially private under pseudo-metric $\v{\rho}(\v{x},\v{y})$.
  \item Under the conditions in Lemma~\ref{lemma:Sto-Lip-Con}, if $\v{\rho}(\v{x},\v{y})\leq (1-\delta)c$ uniformly for all $\v{x}, \v{y}$ for some $\delta\in(0,1)$:
    $$\abs{\bel(B \mid \v{x}) - \bel(B \mid \v{y})} \leq \sqrt{\frac{M}{2} \cdot \max\{\v{\rho}(\v{x},\v{y}),1\}}$$
    where $M=\left(\frac{\kappa}{c}+L_{0}(\frac{1}{1-e^{-\omega}}+1)+\ln{C}\right.$ $\left.+\ln\left(e^{-L_{0}\delta c}(e^{-\omega(1-\delta)}-e^{-\omega})^{-1}+e^{L_{0}(1-\delta)c}\right)\right)C$; constants $\kappa=4.91081$ and $\omega=1.25643$; $C=\prod^{|I|}_{i}C_{i}$; and 
    \begin{eqnarray*}
	    C_i &=& \sup_{\v{x}\in\v{\X}}\frac{p_{\theta^\star_{i,\mathrm{MLE}}}(x_i\mid x_{\pi_i})}{\int_{\Theta_i} p_{\theta_i}(x_i\mid x_{\pi_i}) d\xi(\theta_i)}\enspace,
    \end{eqnarray*}
    the ratio between the maximum and marginal likelihoods of each likelihood function.
    Note that $M=\bigoh{\left(\frac{1}{c}+\ln{C}+L_{0}\right)C}$
    \ie the posterior $\bel$ is
    $\left(0,  \sqrt{\frac{M}{2}}\right)$-differentially private under pseudo-metric $\sqrt{\v{\rho}}$ for $\v{\rho}(\v{x},\v{y})\geq 1$.
  \end{enumerate}
\end{thm}





\subsection{MAP by the Exponential Mechanism}

As an application of the posterior sampler, we
now turn to releasing
MAP point estimates via the exponential mechanism~\cite{MechDesign}, which samples responses from a likelihood exponential in some score function. By selecting a utility function that is maximised by a target non-private mechanism, the exponential mechanism can be used to privately approximate
that target with high utility. It is natural then to select as our utility $\util$ the posterior likelihood $\pri{\cdot|\data}$. 
This $\util$ is maximised by the MAP estimate. 

\begin{algorithm}[htb]
  \begin{algorithmic}[1]
	  \STATE \textbf{Input} data \data; prior $\pri{\cdot}$; appropriate smoothness parameters $c, L, M>0$; parameters distance $r>0$, privacy $\epsilon>0$
    \STATE calculate posterior $\pri{{\param}|\data}$
    \STATE set $\Delta = \begin{cases} \sqrt{L r}\ , & \mbox{if Lipschitz continuous} \\\sqrt{0.5 M}\ , & \mbox{if stochastic Lipschitz}\end{cases}$
    \STATE output $\hat{\param}$ sampled $\propto\exp\left(\frac{\epsilon\pri{\param|\data}}{2\Delta}\right) \pri{\param}$
  \end{algorithmic}
  \caption{Mechanism for MAP Point Estimates \label{alg:map}}
\end{algorithm}

Formally, Algorithm~\ref{alg:map}, under the assumptions of Theorem~\ref{thm:sampler-dp}, outputs response $\param$ with probability proportional to $\exp(\epsilon \util(D,\param) / 2\Delta)$ times a base measure $\mu(\param)$. Here $\Delta$ is a Lipschitz coefficient for $\util$ with sup-norm on responses and pseudo-metric ${\rho}$ on datasets as in the previous section. 
Providing the base measure is non-trivial in general, but for discrete finite outcome spaces can be uniform~\cite{MechDesign}. For our mechanism to be broadly applicable, we can safely take $\mu(\param)$ as $\pri{\param}$.\footnote{In particular the base measure guarantees we have a proper density function: if $\util(D,\param)$ is bounded by $M$, then we have normalising constant $\int_{\param}{\exp(\epsilon \util(D,\param))\mu(\param)d\param}\leq\exp(M\epsilon)<\infty.$}

\begin{figure*}[t!]
\begin{minipage}[t]{.49\textwidth}
\centering
\includegraphics[width=0.95\textwidth]{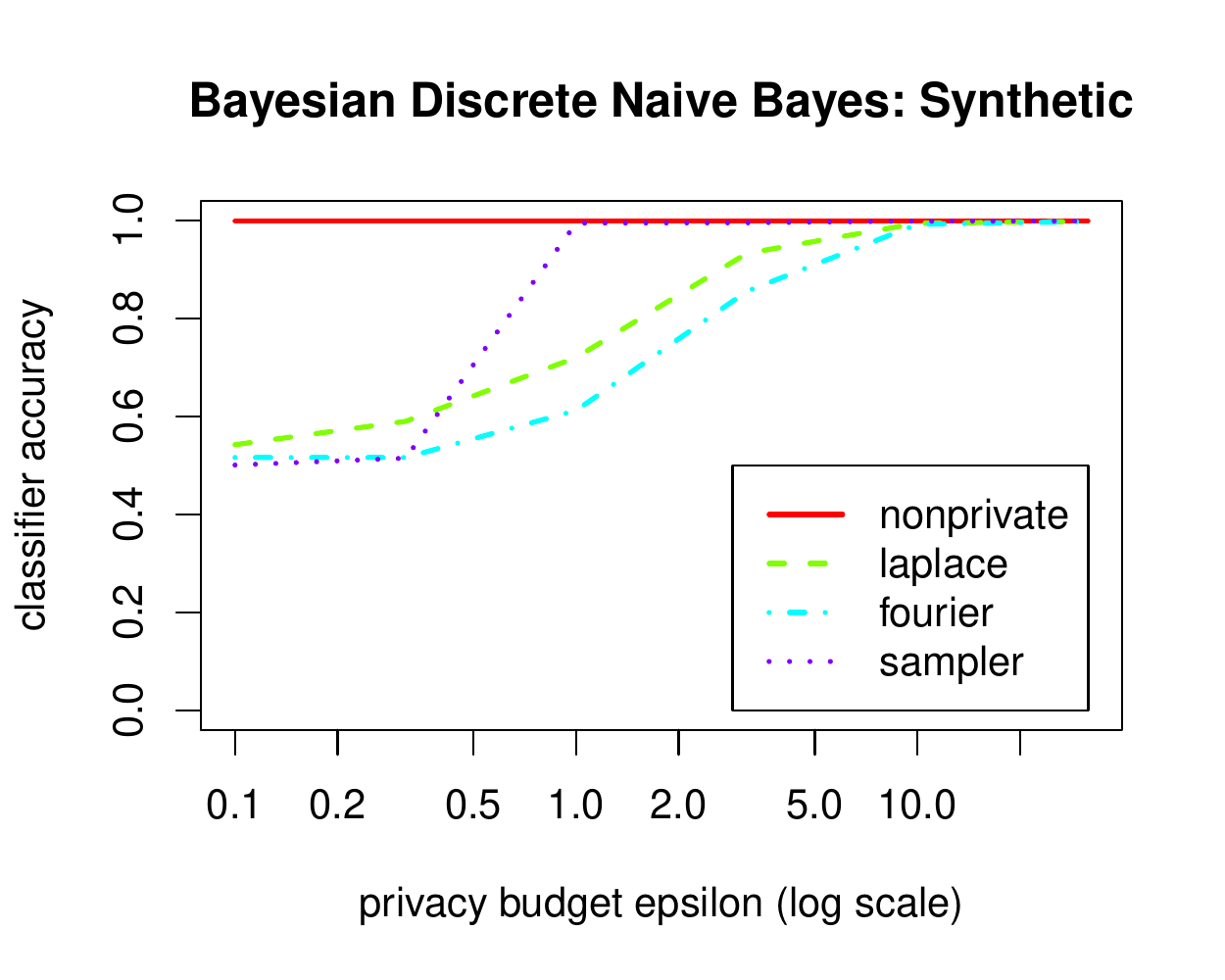}
\caption{Effect on Bayesian na\"ive Bayes predictive-posterior accuracy of varying the privacy level.}
\label{fig:map-bound}
\end{minipage}\hfill
\begin{minipage}[t]{.49\textwidth}
\centering
\includegraphics[width=0.95\linewidth]{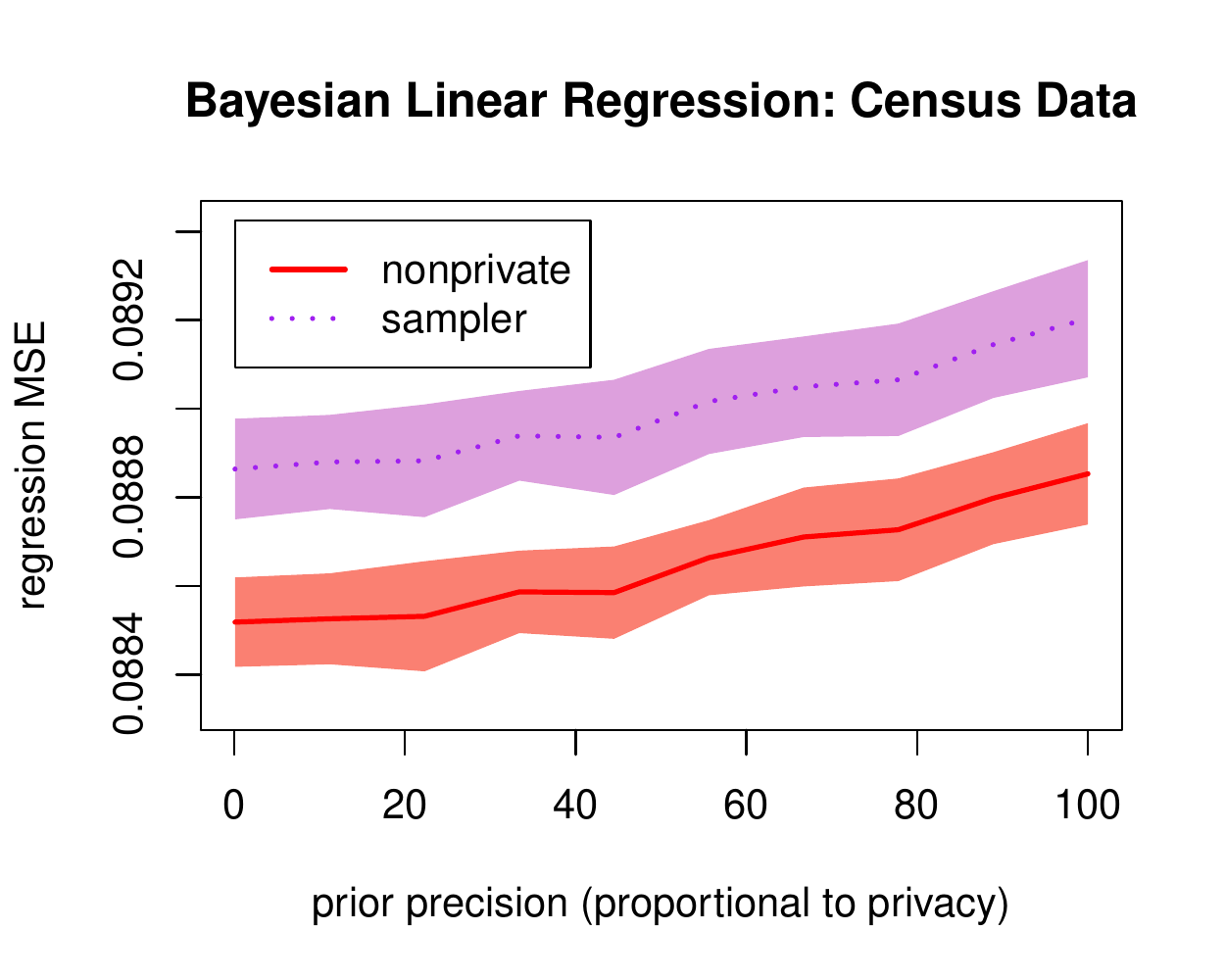}
\caption{Effect on linear regression of varying prior concentration. Bands indicate standard error over repeats.}
\label{fig:census}
\end{minipage}\\
\end{figure*}

\begin{cor}
	Algorithm~\ref{alg:map} preserves $\epsilon$-differential privacy wrt pseudo-metric $\rho$ up to distance $r>0$. 
\end{cor}

\begin{proof}
	The sensitivity of the posterior score function corresponds to the computed $\Delta$~\cite[Theorem 6]{arxiv:robust} under either Lipschitz assumptions. The result then follows from \cite[Theorem 6]{MechDesign}.
\end{proof}

Utility for Algorithm~\ref{alg:map} follows from \cite{MechDesign}, and states that the posterior likelihood of responses is likely to be close to that of the MAP.

\begin{lemma}\label{lemma:map-utility}
Let $\param^\star=\max_\param \pri{\param|\data}$ with maximizer the MAP estimate, and let
$S_{t}=\{\param\in\Theta : \pri{\param|\data} > \param^\star -t\}$ for $t>0$. Then $\Pr(S^{c}_{2t})\leq \exp(-\epsilon t)/\pri{S_{t}}$.
\end{lemma}








\section{Experiments}

Having proposed a number of mechanisms for approximating exact Bayesian inference in the general framework of probabilistic graphical models, we now demonstrate our approaches on two simple, well-known PGMs: the (generative) na\"ive Bayes classifier, and (discriminative) linear regression. This section, with derivations in the
\iftoggle{fullpaper}{Appendix,}{full report~\cite{arxiv:full},}
illustrates how our approaches are applied, and supports our extensive theoretical results with experimental observation. We focus on the trade-off between privacy and utility (accuracy and MSE respectively), which involves the (private) posterior via a predictive posterior distribution in both case studies.

\subsection{Bayesian Discrete Na\"ive Bayes}

An illustrative example for our mechanisms is a Bayesian na\"ive Bayes model on Bernoulli class $Y$ and attribute variables $X_i$, with full conjugate Beta priors. This PGM directly specialises the running Example~\ref{ex:boolean}.
We synthesised data generated from a na\"ive Bayes model, with $16$ features and $1000$ examples. Of these we trained our mechanisms on only $50$ examples, with uniform Beta priors. We formed predictive posteriors for $Y|\v{X}$ from which we thresholded at 0.5 to make classification predictions on the remaining, unseen test data so as to evaluate classification accuracy. The results are reported in Figure~\ref{fig:map-bound}, where average performance is taken over 100 repeats to account for randomness in train/test split, and randomised mechanisms.

\emph{The small size of this data represents a challenge in our setting, since privacy is more difficult to preserve under smaller samples~\cite{Dwork06}.} As expected, privacy incurs a sacrifice to accuracy for all private mechanisms.

For both Laplace mechanisms that perturb posterior updates, note that the
$d$ Boolean attributes and class label (being sole parent to each) yields nodes $|\I|=d+1$ and downward closure size $|\dclose|=2d+2$. Following our generic mechanisms, the noise added to sufficient statistics is independent on training set size, and is similar in scale. $t$ was set for the Fourier approach, so that stealth was achieved 90\% of the time---those times that contributed to the plot. Due to the small increments to cell counts for Fourier, necessary to achieve its \emph{additional stealth property}, we expect a \emph{small decrease to utility which is borne out in Figure~\ref{fig:map-bound}}.



For the posterior sampler mechanism, while we can apply Assumption~\ref{ass:hoelder-measure-observations} to a Bernoulli-Beta pair to obtain a generalised form of $(\epsilon,\delta)$-differential privacy, we wish to compare with our $\epsilon$-differentially-private mechanisms and so choose a route which satisfies Assumption~\ref{ass:hoelder-observations} as detailed in the 
\iftoggle{fullpaper}{Appendix.}{full report~\cite{arxiv:full}.}
We trim the posterior before sampling, so that probabilities are lower-bounded.
Figure~\ref{fig:map-bound} demonstrates that for small $\epsilon$, the minimal probability at which to trim is relatively large resulting in a poor approximate posterior. But past a certain threshold, \emph{the posterior sampler eventually outperforms the other private mechanisms.}





\subsection{Bayesian Linear Regression}
We next explore a system of continuous r.v.'s in Bayesian linear regression, for which our posterior sampler is most appropriate. We model label $Y$ as \iid Gaussian with known-variance and mean a linear function of features, and the linear weights endowed with multivariate Gaussian prior with zero mean and spherical covariance. To satisfy Assumption~\ref{ass:hoelder-observations} we conservatively truncate the Gaussian prior (\cf the
\iftoggle{fullpaper}{Appendix),}{full report~\citenop{arxiv:full}),}
and sample from the resulting truncated posterior; form a predictive posterior; then compute mean squared error. To evaluate our approach we used the U.S. census records dataset from the \emph{Integrated Public Use Microdata Series}~\cite{census} with $370$k records and $14$ demographic features. 
To predict \emph{Annual Income}, we train on $10\%$ data with the remainder for testing. Figure~\ref{fig:census} displays MSE under varying prior precision $b$ (inverse of covariance) and weights with bounded norm $10/\sqrt{b}$ (chosen conservatively). As expected, more concentrated prior (larger $b$) leads to worse MSE for both mechanisms, as stronger priors reduce data influence. Compared with linear regression, private regression suffers only slightly worse MSE. At the same time the posterior sampler enjoys increasing privacy (that is proportional to the bounded norm as given in the
\iftoggle{fullpaper}{Appendix).}{full report).}
\section{Conclusions}

We have presented a suite of mechanisms for differentially-private inference in graphical models, in a Bayesian framework. The first two perturb posterior parameters to achieve privacy. This can be achieved either by performing perturbations in the original parameter domain, or in the frequency domain via a Fourier transform. Our third mechanism relies on the choice of a prior, in combination with posterior sampling. We complement our mechanisms for releasing the posterior, with private MAP point estimators. Throughout we have proved utility and privacy bounds for our mechanisms, which in most cases depend on the \emph{graph structure of the Bayesian network: naturally, conditional independence affects privacy.} We support our new mechanisms and analysis with applications to two concrete models, with experiments exploring the privacy-utility trade-off.





\paragraph{Acknowledgements.}
This work was partially supported by the Swiss National Foundation grant ``Swiss Sense Synergy'' \texttt{CRSII2-154458}\iftoggle{fullpaper}{}{, and by the Australian Research Council (\texttt{DE160100584})}.

\iftoggle{fullpaper}
{
\appendix

\setcounter{secnumdepth}{1}

\section{Proofs for Laplace Mechanism on Posterior Updates}

\setcounter{secnumdepth}{2}

\subsection{Proof of Proposition~\ref{prop:laplace-util-updates}}

Let us denote the event of a Laplace sample exceeding $z>0$ in absolute value as $A_{k}$, $k\in {1, \cdots, 2m}$. Consider the probability of an event that none of the $2m$ \iid Laplace noise we add to each count exceed $z>0$ in absolute value:
\begin{align*}
 1-\Pr[\cup_{k=1}^{2m} \{A_{k}\}]&\geq 1-\sum_{k=1}^{2m}\Pr[A_{i}]\\
 &=1-2m\exp(-z\epsilon/2|\I|)).
 \end{align*}
To make sure this probability is no smaller than $1-\delta$, we need $z$ to be at most to $\frac{2|\I|}{\epsilon}\ln(\frac{2m}{\delta})$.

\subsection{Proof of Theorem~\ref{thm:McDiarmid}}

\begin{appxlem}
(\emph{McDiarmid's inequality})
\label{McDiarmid}
Suppose that random variables $\v{Z}_{1}, \cdots,\v{Z}_{m}\in \mathcal{Z}$ are independent, $f$ is a mapping from $\mathcal{Z}^{m}$ to $\mathbb{R}$. For $\v{z}_{1}, \cdots, \v{z}_{n}, \v{z}'_{k}\in \mathcal{Z}$, if $f$ satisfies
$$|f(\v{z}_{1},..., \v{z}_{m})-f(\v{z}_{1},..., \v{z}_{k-1}, \v{z}'_{k}, \v{z}_{k+1},...,z_{m})|\leq c_{k}$$
Then $$Pr[f(\v{z}_{1},..., \v{z}_{m})-Ef\geq t]\leq \exp\left(\frac{-2t^{2}}{\sum_{k}c^{2}_{k}}\right)$$
\end{appxlem}

To prove Theorem~\ref{thm:McDiarmid}, we need the following statements.

\begin{appxlem}\label{lemma:mean value}
For constants $t\geq 0$ and $a$, $(a+t)\ln(a+t)-a\ln{a}\leq t\ln(a+t)+t$.
\end{appxlem}
\begin{proof}
This follows from applying the mean value theorem to the function $x\ln(x)$ on the interval $[a,a+t].$
\end{proof}

We need to assume that $\alpha_{i,j}$ and $\beta_{i,j}$ are larger than the only turning point of the $\Gamma$ function which is between $1$ and $2$; $\alpha_{i,j}, \beta_{i,j}\geq 2$ is sufficient.\footnote{To cover more priors, we could assume that $\alpha_{i,j}$ is bounded away from zero, and that $\Gamma$ at this parameter is maximum below $2$ and proceed from there for the second case.}

\begin{appxlem}
For $Z_{i,j}\in \mathcal{Z}$,
\begin{align*}
f_{i,j}(\v{Z}_{i,j})&\leq \Delta\alpha_{i,j}\ln(\alpha_{i,j}+\Delta\alpha_{i,j})\\
&+\Delta\beta_{i,j}\ln(\beta_{i,j}+\Delta\beta_{i,j}) \\
&-Z_{i,j}^{(1)}\ln(\alpha_{i,j}+\Delta\alpha_{i,j}-1) \\
&-Z^{(2)}_{i,j}\ln(\beta_{i,j}+\Delta\beta_{i,j}-1)\enspace.
\end{align*}
\label{lemma:difference of RV}
\end{appxlem}
\begin{proof}
By monotonicity of the $\Gamma$ function,
\begin{eqnarray*}
	&&\ln\left(\frac{\Gamma(\alpha_{ij}+\Delta\alpha_{ij})}{\Gamma(\alpha_{ij}+Z_{ij})}\right) \\
	&\leq& \ln\left(\frac{\Gamma(\alpha_{ij}+\Delta\alpha_{ij})}{\Gamma(\alpha_{ij})}\right) \\
																																					     &\leq& \ln\left(\frac{\Gamma(\alpha_{ij})\prod_{r=0}^{\Delta\alpha_{ij}-1}(\alpha_{ij}+r)}{\Gamma(\alpha_{ij})}\right) \\
	&=& \sum^{\Delta\alpha_{ij}-1}_{r=0}\ln(\alpha_{ij}+r)\enspace.
\end{eqnarray*}
and by inequalities  $\ln(x-1)\leq \psi(x)\leq \ln(x)$, we have
\begin{align*}
f_{ij}(\v{Z}_{ij})&\leq\sum^{\Delta\alpha_{ij}-1}_{r=0}\ln(\alpha_{ij}+r)+\sum^{\Delta\beta_{ij}-1}_{r=0}\ln(\beta_{ij}+r)\\
&+(\Delta\alpha_{ij}-Z^{(1)}_{ij})\psi(\alpha_{ij}+\Delta\alpha_{ij})\\
&+(\Delta\beta_{ij}-Z^{(2)}_{ij})\psi(\beta_{ij}+\Delta\beta_{ij})\\
&\leq (\alpha_{ij}+2\Delta\alpha_{ij})\ln(\alpha_{ij}+\Delta\alpha_{ij})\\
&+(\beta_{ij}+2\Delta\beta_{ij})\ln(\beta_{ij}+\Delta\beta_{ij})\\
&-\alpha_{ij}\ln\alpha_{ij}-\beta_{ij}\ln\beta_{ij}-\Delta\alpha_{ij}-\Delta\beta_{ij}\\
&-Z^{(1)}\ln(\alpha_{ij}+\Delta\alpha_{ij}-1)\\
&-Z^{(2)}_{ij}\ln(\beta_{ij}+\Delta\beta_{ij}-1)\\
&\leq \Delta\alpha_{ij}\ln(\alpha_{ij}+\Delta\alpha_{ij})+\Delta\beta_{ij}\ln(\beta_{ij}+\Delta\beta_{ij})\\
&-Z_{ij}^{(1)}\ln(\alpha_{ij}+\Delta\alpha_{ij}-1)\\
&-Z^{(2)}_{ij}\ln(\beta_{ij}+\Delta\beta_{ij}-1)
\end{align*}
The last inequality follows from
\begin{eqnarray*}
	\sum^{\Delta\alpha_{ij}-1}_{r=0}\ln(\alpha_{ij}+r) &<& \int^{\Delta\alpha_{ij}}_{0}\ln(\alpha_{ij}+x)dx \\
&=&(\alpha_{ij}+\Delta\alpha_{ij})\ln(\alpha_{ij}+\Delta\alpha_{ij}) \\
&&-\alpha_{ij}\ln\alpha_{ij}-\Delta\alpha_{ij}.
\end{eqnarray*}
\end{proof}

\begin{appxlem}\label{lemma:Difference of sample}
For $\v{z}_{i,j}, \v{z}'_{i,j}\in \mathcal{Z}$, $f$ satisfies
\begin{eqnarray*}
	&& |f_{i,j}(\v{z}_{i,j})-f_{i,j}({\v{z}'_{i,j}})| \\
	&\leq& (2n+1)[\ln(\alpha_{i,j}+n+1)+(\beta_{i,j}+n+1)]
\end{eqnarray*}
\end{appxlem}
\begin{proof}
\begin{align*}
&\left|f_{ij}(\v{z}_{ij})-f_{ij}({\v{z}'_{ij}})\right| \\
	=&\left|\ln\frac{\Gamma\left(\alpha_{ij}+z'^{(1)}_{ij}\right)}{\Gamma\left(\alpha_{ij}+z^{(1)}_{ij}\right)}+\ln\frac{\Gamma\left(\beta_{ij}+z'^{(2)}_{ij}\right)}{\Gamma\left(\beta_{ij}+z^{(2)}_{ij}\right)}\right.\\
&+\left(z'^{(1)}_{ij}-z^{(1)}_{ij}\right)\psi(\alpha_{ij}+\Delta\alpha_{ij})\\
&\left.\vphantom{\ln\frac{\Gamma(\alpha_{ij}+z'^{(1)}_{ij})}{\Gamma(\alpha_{ij}+z^{(1)}_{ij})}}+\left(z'^{(2)}_{ij}-z^{(2)}_{ij}\right)\psi(\beta_{ij}+\Delta\beta_{ij})\right|\\
\leq& \ln\frac{\Gamma(\alpha_{ij}+n)}{\Gamma(\alpha_{ij})}+\ln\frac{\Gamma(\beta_{ij}+n)}{\Gamma(\beta_{ij})}\\
&+n\ln(\alpha_{ij}+\Delta\alpha_{ij})+n\ln(\beta_{ij}+\Delta\beta_{ij})\\
\leq& \sum^{n}_{r=0}\ln(\alpha_{ij}+r)+\sum^{n}_{r=0}\ln(\beta_{ij}+r)\\
&+n\ln(\alpha_{ij}+\Delta\alpha_{ij})+n\ln(\beta_{ij}+\Delta\beta_{ij})\\
\leq& (\alpha_{ij}+n+1)\ln(\alpha_{ij}+n+1)-\alpha_{ij}\ln\alpha_{ij}\\
&+(\beta_{ij}+n+1)\ln(\beta_{ij}+n+1)-\beta_{ij}\ln\beta_{ij}-2n-2\\
&+n\left(\ln(\alpha_{ij}+n)+\ln(\beta_{ij}+n)\right)\\
\leq& (2n+1)\left(\ln(\alpha_{ij}+n+1)+\ln(\beta_{ij}+n+1)\right)
\end{align*}

\end{proof}


The distribution of $Z^{(1)}_{ij}$ is given by,
$\begin{cases}
\Pr(Z^{(1)}_{ij}=0)=\Pr(Y+\Delta\alpha_{ij}\leq 0)\\
\Pr(Z^{(1)}_{ij}=n)=\Pr(Y+\Delta\alpha_{ij}\geq n)\\
\Pr(Z^{(1)}_{ij}\leq z)=\Pr(Y\leq z-\Delta\alpha_{ij})& z\in (0,n)\\
\end{cases}$

Then we have
\begin{align*}
&\E\left(Z^{(1)}_{ij}\right)\\
=& \int^{n}_{0}[1-F_{Y}(z-\Delta\alpha_{ij})]dz+n\Pr(Y+\Delta\alpha_{ij}\geq n)\\
=&\int_{0}^{\Delta\alpha_{ij}}\left[1-\frac{1}{2}\exp\left(\frac{z-\Delta\alpha_{ij}}{b}\right)\right]dz\\
&+\int_{\Delta\alpha_{ij}}^{n}\frac{1}{2}\exp\left(\frac{\Delta\alpha_{ij}-z}{b}\right)dz+\frac{n}{2}\exp\left(\frac{\Delta\alpha_{ij}-n}{b}\right)\\
=&\Delta\alpha_{ij}+\frac{b}{2}\exp\left(-\frac{\Delta\alpha_{ij}}{b}\right)+\frac{n-b}{2}\exp\left(\frac{\Delta\alpha_{ij}-n}{b}\right)
\end{align*}
By the same argument, the expectation of $Z^{(2)}_{ij}$ is given by $\Delta\beta_{ij}+\frac{b}{2}\exp\left(-\frac{\Delta\beta_{ij}}{b}\right)+\frac{n-b}{2}\exp\left(\frac{\Delta\beta_{ij}-n}{b}\right)$. %

By plugging $\E\left(Z^{(1)}_{i,j}\right)$ and $\E\left(Z^{(2)}_{i,j}\right)$ in Lemma~\ref{lemma:difference of RV} with have
\begin{align*}
\E(f_{ij}(Z_{ij}))&\leq (\alpha_{ij}+2\Delta\alpha_{ij})\ln(\alpha_{ij}+\Delta\alpha_{ij})\\
&+(\beta_{ij}+2\Delta\beta_{ij})\ln(\beta_{ij}+\Delta\beta_{ij})\\
&-\alpha_{ij}\ln\alpha_{ij}-\beta_{ij}\ln\beta_{ij}-\Delta\alpha_{ij}-\Delta\beta_{ij}\\
&-\ln(\alpha_{ij}+\Delta\alpha_{ij}-1)\E\left(Z^{(1)}_{ij}\right)\\
&-\ln(\beta_{ij}+\Delta\beta_{ij}-1)\E\left(Z^{(2)}_{ij}\right)\\
&\leq n\ln[(\alpha_{ij}+\Delta\alpha_{ij})(\beta_{ij}+\Delta\beta_{ij})].
\end{align*}
When $n\geq b$, this can be refined as
$$\bigoh{n\ln{n}}\left[1-\exp\left(-\frac{n\epsilon}{2|\I|}\right)\right]$$
since $\E(Z^{(1)}_{i,j})$ and $\E(Z^{(2)}_{i,j})$ are lower bounded by $\frac{n}{2}\exp(-\frac{n}{b})$.\\
Therefore, by McDiarmid's difference inequality we have
\begin{align*}
\Pr\left[\sum_{i,j}f_{i,j}-\sum_{i,j}\E(f_{i,j})\geq \sqrt{-\frac{1}{2}\ln\delta\sum_{i,j}c_{i,j}}\right]&\leq \delta,
\end{align*}
where $c_{i,j}$ is the RHS in Lemma~\ref{lemma:Difference of sample}.


\subsection{Proof of Theorem~\ref{thm:fourier-utility}}

	We follow the proof of \cite[Theorem 7]{Barak2007}. If $X\sim\lap{b}$ then by the CDF of the Laplace $\prob{|X|>R}=\exp(-R/b)$ where $R>0$. By the union bound for $\{X_j\}_{j\in \dclose}\stackrel{\iid}{\sim}\lap{b}$, we have w.h.p. none is large $\prob{\forall j\in \dclose, |X_j|\leq b\log(|\dclose|/\delta)} \geq 1 - \delta$ for $\delta\in(0,1)$. Since $\|f^j\|_1=2^{k/2}$ for each $j\subseteq\I$ it follows with probability at least $1-\delta$, that $\forall j\in \dclose\backslash\{\emptyset\}, \left\|z_j f^j - \langle f^j, h\rangle f^j \right\|_1\leq \frac{2|\dclose|}{\epsilon}\log\frac{|\dclose|}{\delta}$. For $f^{\v{0}}$ the additional increment comes at an additional cost of $4 t |\dclose|^2 / \epsilon$. Putting everything together, we note that $2^{|\pi_i|+1}$ Fourier coefficients represent $h_i$ including $f^{\v{0}}$.





\section{Posterior Sampling}

\subsection{Proof of Lemma~\ref{lemma:Lip-Con}}
\begin{eqnarray*}
d(p_{\v{\theta}}(x), p_{\v{\theta}}(y))&=&\left|\log\prod^{|\I|}_{i=1}\dfrac{p_{\v{\theta}}(x_{i}|x_{\pi_{i}})}{p_{\v{\theta}}(y_{i}|y_{\pi_{i}})}\right|\\
&\leq &\sum^{|\I|}_{i=1}\left|\log\dfrac{p_{\v{\theta}}(x_{i}|x_{\pi_{i}})}{p_{\v{\theta}}(y_{i}|y_{\pi_{i}})}\right| \\ 
&=& \sum^{|\I|}_{i=1} d(p_{\v{\theta}}(x_{i}|x_{\pi_{i}}), p_{\v{\theta}}(y_{i}|y_{\pi_{i}}))\\
&\leq & \sum^{|\I|}_{i=1} L_{i}\rho_{i}(x_{i}, y_{i})\\
&\leq & \| \v{L}\|_{\infty} \| \v{\rho}(\v{x},\v{y})\|_{1}.
\end{eqnarray*}


\subsection{Proof of Lemma~\ref{lemma:Sto-Lip-Con}}
Define
\begin{eqnarray*}
{\Theta}_{i, L} &=& \left\{\v{\theta}\in \Theta:
\sup_{x,y\in \mathcal{X}}\{d(p_{\v{\theta}}(x_{i}|x_{\pi_{i}}), p_{\v{\theta}}(y_{i}|y_{\pi_{i}})) \right. \\
&& \ \ \ \ \ \left.\vphantom{\sup_{x,y\in \mathcal{X}}} - L\rho_{i}(x_{i}, y_{i}) \} \leq 0 \right\}\ .
\end{eqnarray*}
By taking $\rho(x,y)=\sum_{i}\rho_{i}(x_{i}, y _{i})$, we have
\begin{align*}
&\bigcap^{|\I|}_{i=1}\tilde{\Theta}_{i,L}\\
=&\left\{\v{\theta} \in \Theta: \sup_{x_{i},y_{i}\in \mathcal{X}_{i}}\{d(p_{\v{\theta}}(x_{i}|x_{\pi_{i}}), p_{\v{\theta}}(y_{i}|y_{\pi_{i}}))\right.\\
 &\left. \vphantom{\sup_{x_{i},y_{i}\in \mathcal{X}_{i}}} \ \ \ \leq L\rho_{i}(x_{i}, y_{i})\}, \forall i\in\I\right\}\\
	\subseteq& \left\{\v{\theta} \in \Theta: \sup_{x_{i},y_{i}\in \mathcal{X}_{i}}\left\{\sum_{i=1}^{|\I|}d(p_{\v{\theta}}(x_{i}|x_{\pi_{i}}), p_{\v{\theta}}(y_{i}|y_{\pi_{i}}))\right.\right.\\
		 &\left.\left. \ \ \ \ \ \vphantom{\sum_{i=1}^{|\I|}}\leq L\sum_{i=1}^{|\I|}\rho_{i}(x_{i}, y_{i})\right\}\right\}\\
	\subseteq& \left\{\v{\theta} \in \Theta: \sup\{d(p_{\v{\theta}}(x), p_{\v{\theta}}(y)) - L\rho(x, y)\}\leq 0\right\} \\
= &\Theta_{L}
\end{align*}

Therefore, we have that the set of $\v{\theta} \in \Theta$ satisfying the Stochastic Lipschitz continuity for conditional likelihood of every $i\in \I$ in the Bayesian network is a subset of the set of $\v{\theta}$ satisfying the global Stochastic Lipschitz continuity for same $L$.

Note that $(\bigcap^{|\I|}_{i=1}\Theta_{i,L})^{c}=\bigcup^{|\I|}_{i=1}(\Theta_{i,L})^{c}$ and $\xi((\Theta_{i,L})^{c})=1-\xi(\Theta_{i,L})\leq e^{-c_{i}L}$. Then we have
$$\xi[(\cap_{i=1}^{|\I|}\Theta_{i,L})^{c}]\leq \sum_{i=1}^{|\I|}\xi(\Theta_{i,L})^{c})\leq\sum_{i=1}^{|\I|}e^{-c_{i}L}.$$
Therefore, we have
$$\xi(\Theta_{L})\geq\xi(\cap^{|\I|}_{i=1}\Theta_{i,L})\geq 1-\sum^{|\I|}_{i=1}e^{-c_{i}L}\geq 1-Ne^{-{\min_{i}{c_{i}}}L}.$$
Take $c'=\ln{|\I|}\min\{c_{i}\}_{i=1}^{|\I|}$, we have $\xi(\Theta_{L})\geq 1-e^{-c'L}$.

\section{Bayesian Na\"ive Bayes}
\label{sec:nb}

We review the derivation of the na\"ive Bayes predictive posterior for two cases applied in our experiments.

\subsection{Closed-Form Beta-Bernoulli}

When the r.v.'s are all Bernoulli's with Beta conjugate priors:
$$ \Pr(Y=y|\v{X}=\v{x}) 
	\propto \int_\Theta \lih{\theta}{y}\prod_{i=1}^d \lih{\theta}{x_i|y} \pri{\theta} d\theta .$$

	The integral decouples into the product of (where $\alpha,\beta$ refer to the $y$ posterior)
\begin{eqnarray*}
	&& \int_\Theta \lih{\theta}{y}\pri{\theta} d\theta \\
	&=& \int_0^1 \frac{\theta^{\alpha+y-1}(1-\theta)^{\beta+(1-y)-1}}{B(\alpha,\beta)} d\theta \\
	&=& \frac{B(\alpha+y,\beta+1-y)}{B(\alpha,\beta)} \\
	&& \ \ \ \times \ \int_0^1 \frac{\theta^{\alpha+y-1}(1-\theta)^{\beta+(1-y)-1}}{B(\alpha+y,\beta+1-y)} d\theta \\
	&=& \frac{B(\alpha+y,\beta+1-y)}{B(\alpha,\beta)} \\
	&=& \frac{\Gamma(\alpha+y)\Gamma(\beta+1-y)}{\Gamma(\alpha+\beta+1)} \frac{\Gamma(\alpha+\beta)}{\Gamma(\alpha)\Gamma(\beta)} \\
	&=& \frac{\alpha^y \beta^{1-y}}{\alpha+\beta}\enspace.
\end{eqnarray*}
and terms (where $\alpha,\beta$ refer to the $x_i\mid y$ posterior)
\begin{eqnarray*}
	&& \int_0^1 \frac{\theta^{\alpha+x_i-1}(1-\theta)^{\beta+(1-x_i)-1}}{B(\alpha,\beta)} d\theta \\
	&=& \frac{\alpha^{x_i} \beta^{1-x_i}}{\alpha+\beta}\enspace,
\end{eqnarray*}
computed in the same way.

\subsection{Sampling}

Given an empirical CDF sampled from our posterior sampling mechanism, we can approximate by posterior sampling:
\begin{itemize}
	\item Repeat many times for both $y=0, y=1$:
		\begin{itemize}
			\item Sample $\hat{\theta}_y, \hat{\theta}_{x_1,y},\ldots,\hat{\theta}_{x_d,y}$
			\item Plug-in the sampled parameters and fixed r.v.'s into the product of densities to obtain an unnormalised probability estimate
		\end{itemize}
	\item Average the obtained estimates, for each $y=0,y=1$
	\item Normalise
\end{itemize}

We modify the above slightly so that we sample from a truncated posterior. This allows us to assume a minimal probability $\omega$ assigned to any sub-event in the na\"ive Bayes network, so that the joint distribution satisfies Assumption~\ref{ass:hoelder-observations}. Trivially in particular this yields a differential privacy level given by $\epsilon=2\log(1/\omega)$. Given a desired privacy budget $\epsilon$ we can therefore select $\omega=\exp(-\epsilon/2)$. We then simply rejection sample when sampling above, to obtain samples from the truncated posterior. This is the posterior sampler algorithm used in the na\"ive Bayes experiments.




\section{Bayesian Linear Regression}
Let us denote a set of observations $D=\{(x_{1},y_{1}),\ldots, (x_{n},y_{n})\}$ where $x_{i}=(x^{(1)}_{i},\ldots,x^{(d)}_{i})\in\mathbb{R}^{d}, y_{i}\in\mathbb{R}$. In the model we assume that $Y_{i}$ are independent given $x w$.
Recall that a normal linear regression model with i.i.d Gaussian noise is given as follows,
$$y_{i}=\sum^{d}_{j=1}x^{(j)}_{i}{w^{(j)}}+\epsilon_{i}, \epsilon_{i}\thicksim N(0, \sigma^{2}).$$
The normal likelihood function, as a product of likelihoods for each of the individual components of $y=(y_{1},\ldots, y_{n})$, is given by
$$p_{w}(y|x,w;\sigma^{2})=\frac{1}{(2\pi\sigma^{2})^{n/2}}e^{-\frac{1}{2\sigma^{2}}(y-xw)^{T}(y-xw)}.$$
Given observations $D$, we are interested in computing the sensitivity (in terms of data/observation) of this likelihood, that is $\sup_{w,D,D'}|\ln\frac{p_{w}(D)}{p_{w}(D')}|$.
Note that $\ln\frac{p_{w}(D)}{p_{w}(D')}=\ln\frac{\prod_{i} p_{w}(x_{i},y_{i})}{\prod_{i} p_{w}(x'_{i},y'_{i})}=\sum_{i}\ln\frac{p_{w}(x_{i},y_{i})}{p_{w}(x'_{i},y'_{i})}$.

For simplicity, assume that the precision of $Y$ is 1. let $f_{w}(D)$ denote the log-likelihood, we have
 $$|f_{w}(D)-f_{w}(D')|\leq\sum_{i}|f_{w}(x_{i},y_{i})-f_{w}(x'_{i},y'_{i}).|$$
 Note that by mean value theorem, we have
 \begin{eqnarray*}
	 && f_{w}(x_{i},y_{i})-f_{w}(x'_{i},y'_{i}) \\
	 &=& \nabla f_{w}((1-c)(x^{(1)}_{i},\ldots, x^{(d)}_{i}, y_{i})\\
  &&\ \ \ +c(x'^{(1)}_{i},\ldots, x'^{(d)}_{i}, y'_{i})) \\
	 &&\ \ \ \cdot (x^{(1)}_{i}-x'^{(1)}_{i},\ldots, x^{(d)}_{i}-x'^{(d)}_{i}, y_{i}-y'_{i})
 \end{eqnarray*}
 Therefore by Cauchy-Schwarz inequality we have:
 \begin{eqnarray*}
	 \Delta f_{w}(x_{i}, y_{i}) &\leq& ||\nabla f_{w}||_{2} ||(\Delta x_{i}, \Delta y_{i})||_{2} \\
 \Delta f_{w}(D, D') &\leq& \sum^{n}_{i=1} \Delta f_{w}(x_{i}, y_{i}) \\
					  &\leq& ||\nabla f_{w}||_{2}\sum^{n}_{i=1}||(\Delta x_{i}, \Delta y_{i})||_{2}
 \end{eqnarray*}
Note that
$$df_{w}(x_{i},y_{i}) / dx^{(j)}_{i} =\frac{1}{2\sigma^{2}}(x^{(j)}_{i}w^{T}w-y_{i}w^{(j)})$$
$$df_{w}(x_{i},y_{i}) / dy_{i}=\frac{1}{2\sigma^{2}}(y_{i}-w^{T}x^{T}_{i})$$


Recall that in linear regression, it is common to assume that every tuple $(x_{i}, y_{i})$ in the database satisfies $||x_{i}||_{2}\leq 1$ and $||y_{i}||_{2}\leq 1$, 
we have
\begin{align*}
	&||\nabla f_{w}||_{2} \\
	\leq&\frac{1}{2\sigma^{2}}\sum^{n}_{i=1}\left(y_{i}-w^{T}x^{T}_{i}+\sum^{d}_{j=1}\left(x^{(j)}_{i}||w||_{2}-y_{i}w^{(j)}\right)\right)\\
\leq& \frac{n}{2\sigma^{2}}(1+2||w||_{1}+d||w||_{2})\\
\leq& \frac{n}{2\sigma^{2}}(1+(d+2)||w||_{1})
\end{align*}
Hence the log-likelihood of normal regression satisfies Assumption $1$ for $\rho(D,D')=\sum^{n}_{i=1}||(\Delta x_{i}, \Delta y_{i})||_{2}$ under the condition that $w$ is bounded.

For normal linear regression with bounded $w$, it is natural to choose a prior of $w$ with truncated normal density, that is
$$p(w)\propto N(0, \Lambda^{-1})\mathbf{1}\{||w||_{2}\leq 1\}$$
(In experiments we vary the norm bound for truncation with $\Lambda$. Our argument extends immediately.) As we show below, this truncated normal prior is still a conjugate prior for Gaussian likelihood.
\begin{appxlem}
The truncated Gaussian prior and the Gaussian likelihood of linear regression is a conjugate pair and the resulted posterior is a truncated Gaussian distribution.
\end{appxlem}
\begin{proof}
By Bayes's rule,
\begin{eqnarray*}
p(w|D)&\propto& p(D|w)p(w)\\
&\propto& N(w| \mu_{n}, \Sigma_{n})\mathbf{1}\{||w||_{2}\leq 1\}
\end{eqnarray*}
where $\mu_{n}=(X^{T}X+\sigma^{2}\Lambda)^{-1}X^{T}y$ and $\Sigma_{n}=\sigma^{2}(X^{T}X+\sigma^{2}\Lambda)^{-1}$.
\end{proof}

Therefore the posterior BAPS (Bayesian Posterior Sampling) on  $p(w|D)$ is $2L(w)$-differentially private, where $L(w)=\frac{n}{2\sigma^{2}}(1+2||w||_{1}+d||w||_{2})$.
}{}



\end{document}